\documentclass{article}
\usepackage[numbers]{natbib}

\usepackage{authblk}
\usepackage[margin=1in]{geometry}
\pdfoutput=1

\usepackage[T1]{fontenc}    % use 8-bit T1 fonts
\usepackage{hyperref}       % hyperlinks

\usepackage{url}            % simple URL typesetting
\usepackage{booktabs}       % professional-quality tables
\usepackage{amsfonts}       % blackboard math symbols
\usepackage{nicefrac}       % compact symbols for 1/2, etc.
\usepackage{microtype}      % microtypography
\usepackage{array}

\usepackage{graphicx}

\usepackage{algorithm}
\usepackage{algorithmic}

\usepackage{amssymb}
\usepackage{multirow}
\usepackage{amsmath}
\usepackage{amsthm}
\usepackage{color}
\usepackage{mathrsfs}
\usepackage{graphicx}
\usepackage{caption, subcaption}
\usepackage{xspace}
\usepackage{float}
\usepackage{wrapfig}
\usepackage{enumitem}
\setitemize{noitemsep,topsep=0pt,parsep=2pt,partopsep=0pt}

\mathchardef\mhyphen="2D

\newcommand{\norm}[1]{{ \left\lVert#1\right\rVert }}
\newcommand{\normb}[1]{\norm{#1}_2}
\newcommand{\normbs}[1]{\normb{#1}^2}

\newcommand{\vertiii}[1]{{\left\vert\kern-0.25ex\left\vert\kern-0.25ex\left\vert #1
    \right\vert\kern-0.25ex\right\vert\kern-0.25ex\right\vert}}

\newcommand{\vect}[1]{{\boldsymbol{#1}}}
\def\bbeta{\vect{\beta}}

\def\bu{{\mathbf{u}}}
\def\bv{{\mathbf{v}}}

\def\bx{{\mathbf{x}}}
\def\by{{\mathbf{y}}}
\def\bz{{\mathbf{z}}}
\def\bC{{\mathbf{C}}}

\def\bI{{\mathbf{I}}}

\def\bP{{\mathbf{P}}}

\def\bT{{\mathbf{T}}}

\def\bX{{\mathbf{X}}}

\def\0{{\mathbf{0}}}

\def\bbE{{\mathbb{E}}}

\def\bbP{{\mathbb{P}}}

\def\bbR{{\mathbb{R}}}

\def\sfA{\mathsf{A}}
\def\sfB{\mathsf{B}}
\def\sfC{\mathsf{C}}

\def\sfG{\mathsf{G}}

\def\sfL{\mathsf{L}}

\def\sfS{\mathsf{S}}
\def\sfT{\mathsf{T}}
\def\sfU{\mathsf{U}}

\newcommand{\inprod}[2]{\langle#1, #2\rangle}

\newtheorem*{rep@theorem}{\rep@title}
\newcommand{\newreptheorem}[2]{%
\newenvironment{rep#1}[1]{%
 \def\rep@title{#2 \ref{##1}}%
 \begin{rep@theorem}}%
 {\end{rep@theorem}}}
\newreptheorem{lemma}{Lemma'}
\newreptheorem{definition}{Definition'}
\newreptheorem{proposition}{Proposition'}
\newreptheorem{theorem}{Theorem'}

\newtheorem{theorem}{Theorem}
\newtheorem{definition}{Definition}
\newtheorem{corollary}{Corollary}
\newtheorem{lemma}{Lemma}
\renewcommand{\text}[1]{{\textnormal{#1}}}

\DeclareMathOperator*{\argmax}{arg\,max}

\newcommand{\greedy}{\textsc{Greedy}}
\newcommand{\stogreedy}{\textsc{StochasticGreedy}}
\newcommand{\distgreedy}{\textsc{DistributedGreedy}}

\title{Scalable Greedy Feature Selection via Weak Submodularity}

\author[1]{Rajiv Khanna}
\author[1]{Ethan R. Elenberg}
\author[1]{Alexandros G. Dimakis}
\author[2]{Sahand Negahban}
\author[1]{Joydeep Ghosh}

\affil[1]{Department of Electrical and Computer Engineering \authorcr The University of Texas at Austin \authorcr \texttt{\{rajivak,\,elenberg\}@utexas.edu}, \texttt{dimakis@austin.utexas.edu}, \texttt{jghosh@utexas.edu}}
\affil[2]{Department of Statistics \authorcr Yale Univeristy \authorcr \texttt{sahand.negahban@yale.edu}}

\begin{document}
\date{}
\maketitle

\begin{abstract}Greedy algorithms are widely used for problems in machine learning such as feature selection and set function optimization. 
Unfortunately, for large datasets, the running time of even greedy algorithms can be quite high. 
This is because for each greedy step we need to refit a model or calculate a function using the previously selected choices and the new candidate. 

Two algorithms that are faster approximations to the greedy forward selection were introduced recently~\citep{Mirzasoleiman:2013vi,Mirzasoleiman:2015th}.
They achieve better performance by exploiting distributed computation and stochastic evaluation respectively.
Both algorithms have provable performance guarantees for submodular functions.

In this paper we show that divergent from previously held opinion, submodularity is not required to obtain approximation guarantees for these two algorithms.
Specifically, we show that a generalized concept of weak submodularity suffices to give multiplicative approximation guarantees. Our result extends the applicability of these algorithms to a larger class of functions.
Furthermore, we show that a bounded submodularity ratio can be used to provide data dependent bounds that can sometimes be tighter also for submodular functions. We empirically validate our work by showing superior performance of fast greedy approximations versus several established baselines on artificial and real datasets. 
\end{abstract}
\section{Introduction}

Consider the problem of sparse linear regression:
\begin{equation}
\label{eqn:linearRegression}
\min_\bbeta \|\by - \bX \bbeta\|^2_2 \,\,\text{   s.t.  } \|\bbeta\|_0 \leq k,
\end{equation}

where $\bX \in \bbR^{n \times d}$ is the design matrix (also called feature matrix) for $n$ samples and $d$ features
and $\by \in \bbR^n$ is the corresponding vector of $n$ responses or observations. We assume without loss of generality that the columns of the matrix $\bX$ are normalized to $1$, and the response vector is also normalized.

Given a subset $\sfS$ of features (denoted by $\bX_\sfS$), it is easy to find the best regression coefficients by projecting $\by$ on the span of $\bX_\sfS$. The $R^2$ statistic (\textit{coefficient of determination}) measures the proportion of the variance explained by 
this subset. Sparse regression can be therefore seen as maximizing a \textit{set function} $R^2(\sfS)$: for any given set of columns $\sfS$ , $R^2(\sfS)$ measures how well these columns explain the observations $\by$. 
 
This set function $R^2(\sfS)$ is \textit{monotone} increasing: including more features can only increase the explained variance\footnote{However, \emph{adjusted} $R^2$ is not monotonic.}. 
Solving sparsity constrained maximization of $R^2$ would involve searching over all subsets of size up to $k$ and selecting the one that maximizes $R^2$. 
This combinatorial optimization problem is unfortunately NP-Hard~\citep{Kempe:2011ue} . A widely used approach is to use 
greedy forward selection: select one feature at a time, greedily choosing the one that maximizes $R^2$ in the next step.
Orthogonal Matching Pursuit and variations~\citep{Tropp:2004gc}~\citep{Needell08cosamp:iterative} can also be seen as approximate accelerated greedy forward selection algorithms that avoid re-fitting the model for each candidate vector at each step. 

When maximizing set functions using the greedy algorithm it is natural to consider the framework of submodularity. 
In a classical result, \citet{Nemhauser:1978ck} show that for submodular monotone functions, the greedy $k$-sparse solution is within $ (1 - \frac{1}{e})$ of the optimal $k$-sparse solution, \emph{i.e.}, greedy gives a constant multiplicative factor approximation guarantee. The concept of submodularity has led to several effective greedy solutions to problems such as sparse prediction~\citep{koyejo2014}, sparse factor analysis~\citep{khanna2015}, model interpretation~\citep{kim:2016ee}, etc.
 
Unfortunately, it is easy to construct counterexamples to show that $R^2(\sfS)$ is not submodular~\citep{Kempe:2011ue,Elenberg:2016tq}. In their breakthrough paper \citet{Kempe:2011ue} show that if the design matrix $\bX$ satisfies the Restricted Isometry Property (RIP), then the set function $R^2(\sfS)$ satisfies a weakened form of submodularity. This weak form of submodularity is obtained by bounding a quantity called a submodularity ratio $\gamma$ defined subsequently. The authors further showed that a bounded submodularity ratio $\gamma$ is sufficient for Nemhauser's proof to go through 
with a relaxed approximation constant (that depends on $\gamma$). Therefore, even weak submodularity implies a constant factor 
approximation for greedy and RIP implies weak submodularity.

The weak submodularity framework was recently extended beyond linear regression by \citet{Elenberg:2016tq}
to concave functions (for example, likelihood functions of generalized linear models). This generalization of the RIP condition is called Restricted Strong Convexity (RSC) and the general result is that 
RSC implies weak submodularity for the set function obtained from the likelihood of the model restricted to subsets of features. 
This shows that greedy feature selection can obtain constant factor approximation guarantees in a very general setting, similar to results obtained by Lasso but without further statistical assumptions.

Running the greedy algorithm can be computationally expensive for large datasets. This is because for each greedy step we need to 
refit the model using the previously selected choices and the new candidate. This has led to the development of faster variants. For example, $\distgreedy$~\citep{Mirzasoleiman:2013vi} distributes the computational effort across available machines, and $\stogreedy$~\citep{Mirzasoleiman:2015th} exploits a stochastic greedy step. Both algorithms have provable performance guarantees for submodular functions.

In this paper, we show that submodularity is not required to obtain approximation guarantees for these two algorithms and that 
a bounded submodularity ratio suffices. This extends the scope of these algorithms to significantly larger class of functions like sparse linear regression with RIP design matrices. Furthermore, as we shall discuss, submodularity ratio can be used to provide data dependent bounds. This implies that one can sometimes obtain tighter guarantees also for submodular functions.

Our contributions are as follows: (1) We analyze and obtain approximation guarantees for $\distgreedy$ and $\stogreedy$ using the submodularity ratio extending the scope of their application to non-submodular functions, \\
\noindent  (2) We show that the submodularity ratio can give tighter data dependent bounds even for submodular functions,\\
\noindent  (3) We derive explicit bounds for both $\distgreedy$ and  $\stogreedy$
for the special case of sparse linear regression. \\
\noindent (4) We derive bounds for sparse support selection for general concave functions. \\
\noindent (5) We also present empirical evaluations of these algorithms vs several established baselines on simulated and real world datasets.

\paragraph{Related Work}
~\citet{Kempe:2011ue} defined submodularity ratio, and showed that for any function that has its submodularity ratio bounded away from $0$, one can provide appropriate greedy guarantees.~\citet{Elenberg:2016tq} used the concept to derive a new relationship between submodularity and convexity, specifically stating that the Restricted Strong Concavity can be used to bound the submodularity ratio. This results in providing bounds for greedy support selection for general concave functions. 

Another notion of approximate additive submodularity was explored by~\citet{Krause:2010tj} for the problem of dictionary selection. This was, however, superceded by~\citep{Kempe:2011ue} who showed that submodularity ratio provides tighter approximation bounds.~\citet{Horel:2016mas} consider another generalization from submodular functions -- $\varepsilon$-approximate submodular functions which are functions within $\varepsilon$ of some submodular function, and provide approximation guarantees for greedy maximization.

The $\distgreedy$ algorithm was introduced by~\citet{Mirzasoleiman:2013vi}. They provide deterministic bounds for any arbitrary distribution of data onto the individual machines.~\citet{Barbosa:2015vq} showed that for sparsity constraints, and under the assumption that the data is split uniformly at random to all the machines, one can obtain a $\frac{1}{2}(1 - \nicefrac{1}{e})$ guarantee in expectation.~\citet{kumar2010} also provide distributed algorithms for maximizing
a monotone submodular function subject to a sparsity constraint. They extend
 the Threshold Greedy algorithm of~\citet{gupta2010} by augmenting it with a sample and prune strategy. It runs the Threshold Greedy algorithm on a subset of data to obtain a candidate solution. The latter is then used to prune the remaining data to reduce its size. This process is repeated a constant number of times, and the algorithm provides a constant factor guarantee

\citet{Altschuler2016} provide approximation guarantees for greedy selection variants for column subset selection. They do not use the submodularity framework, and their results are not directly applicable or useful for other problem settings. Similarly,~\citet{farahatEGK13} also use greedy column subset selection. However, their focus is not towards obtaining approximation guarantees, but rather on more efficient algorithmic implementation.

\section{Background}
\emph{Notation:} We represent vectors as small letter bolds e.g. $\bu$. Matrices are represented by capital bolds e.g. $\bX, \bT$.  Matrix or vector transposes are represented by superscript $\bX^\top$. Identity matrices of size $s$ are represented by $\bI_{s}$, or simply $\bI$ when the dimensions are obvious. $\mathbf{1} (\mathbf{0})$ is a column vector of all ones (zeroes). Sets are represented by sans serif fonts e.g. $\sfS$, complement of a set $\sfS$ is $\sfS^c$. For a vector $\mathbf{u} \in \bbR^d$, and a set $\sfS$ of support dimensions with $|\sfS|=k, k\le d$, $\mathbf{u}_\sfS \in \bbR^k$ denotes subvector of $\mathbf{u}$ supported on $\sfS$. Similarly, for a matrix $\bX \in \bbR^{n\times d}$, $\bX_{\sfS} \in \bbR^{k\times k}$ denotes the submatrix supported on $\sfS$. We denote $\{1,2,\ldots, d\} $ as  $[d]$.

Throughout this manuscript, we assume the set function $f(\cdot)$ is monotone. Our goal is to maximize $f(\cdot)$ under a cardinality constraint : 

\begin{equation}
\label{eq:maxf} \max_{| \sfS| \leq k } f(\sfS).
\end{equation}

We begin by defining the submodularity ratio of a set function $f(\cdot)$.
\begin{definition}[Submodularity Ratio~\citep{Kempe:2011ue} ]\label{def:submodularityRatio}
	Let $\sfS, \sfL \subset [d]$ be two disjoint sets, and $f : [d] \rightarrow \bbR$. The submodularity ratio for $\sfS$ with respect to $\sfL$ is given by
	\begin{align}
	\gamma_{\sfL,\sfS} := \frac{\sum_{j\in \sfS} \left[f(\sfL \cup \{j\}) - f(\sfL) \right]}{f(\sfL \cup \sfS) - f(\sfL)} \label{eq:submodDef} .
	\end{align}
	The submodularity ratio of a set $\sfU$ with respect to an integer $k$ is given by
	\begin{align}
	\gamma_{\sfU,k} := \min_{\substack{\sfL,\sfS :  \sfL \cap \sfS = \emptyset ,\\ \sfL \subseteq \sfU,  |\sfS| \leq k}} \gamma_{\sfL,\sfS} .
	\end{align}
\end{definition}

It is straightforward to show that $f$ is submodular if and only if $\gamma_{\sfL,\sfS} \geq 1$ for all sets $\sfL$ and $\sfS$. Generalizing to the functions with $0<\gamma_{\sfL,\sfS} \leq 1$ provides a notion of \textit{weak submodularity}~\citep{Elenberg:2016tq}.
%wherein even though the function may be submodular, it still provides approximation guarantees for the greedy algorithm. 
For weakly submodular functions, even though the function may not be submodular, it still provides approximation guarantees for the greedy algorithm. 
We use the submodularity ratio to provide new bounds for $\distgreedy$ and $\stogreedy$, thereby generalizing these algorithms to non-submodular functions.

\subsection{Greedy Selection}
We briefly go over the classic greedy algorithm for subset selection. A greedy approach to optimizing a set function is myopic -- the algorithm chooses the element from the available choices that gives the largest \emph{incremental} gain for the set of choices previously made. The algorithm is illustrated in Algorithm~\ref{algo:greedy}. The algorithm makes $k$ \emph{outer} iterations, where $k$ is the desired sparsity. Each iteration is a full pass over the remaining candidate choices, wherein the marginal gain is calculated for each remaining candidate choice. Thus the greedy algorithm has the computational complexity of $O(dk)$ calls to the function evaluation oracle.

\begin{algorithm}[h]
\caption{\greedy($\sfS$, $k$)}
\label{algo:greedy}
\begin{algorithmic}[1]
\STATE Input: sparsity $k$, available choices $\sfS$
\STATE $\sfA_0 = \emptyset$
\FOR {$ i \in 0 \ldots (k-1)$ }
\STATE $s$ = $\argmax_{j \in \sfS \backslash \sfA_{i}} f(\sfA_{i} \cup \{j\}) - f(\sfA_i)$
\STATE $\sfA_{i+1} = \sfA_i \cup \{s\}$ 
\ENDFOR
\STATE return $\sfA_k$
\end{algorithmic}
\end{algorithm}

For a large $d$, the linear scaling of the greedy algorithm for a fixed $k$ may be prohibitive. As such, algorithms that scale sublinearly are useful for truly large scale selections. The $\distgreedy$ algorithm, for example, achieves this sublinear scaling by making use of multiple machines. The data is split uniformly at random to $l$ machines. Each machine then performs its own independent greedy selection (Algorithm~\ref{algo:greedy}), and outputs a $k$ sized solution. All of the greedy solutions are collated by a central machine, which performs another round of the greedy selections to output the final solution. The algorithm is illustrated in Algorithm~\ref{algo:distributedgreedy}, and is analyzed in Section~\ref{sec:distributed}. It has a computational complexity of $O(dk/l)$ The algorithm is easy to implement in parallel or within a distributed computing framework \emph{e.g.} MapReduce.

\begin{algorithm}[h]
\caption{\textsc{DistributedGreedy}($l$, $k$,$\{\sfA_j\} $)}
\label{algo:distributedgreedy}
\begin{algorithmic}[1]
%\STATE Input: sparsity $k$, number of parallel solvers $l$, partition $\{\sfp_i\} $ of $[p]$
%\STATE $\sfS_i \leftarrow $ \func($\sfp_i$, $k$) $\forall i \in [m]$
\STATE Input: sparsity $k$, number of parallel solvers $l$, partition $\{\sfA_j\} $ of the set of available choices $\sfA$
\STATE $\sfG_i \leftarrow $ \greedy($\sfA_j$, $k$) $\forall j \in [l]$
\STATE $\sfG \leftarrow$ \greedy($\cup_j \sfG_j$, $k$)
\STATE $\sfG_{\text{max}} \leftarrow \arg\max_{\sfG_j} f(\sfG_j) $
\STATE return $\arg\max f(\sfG), f( \sfG_\text{max})$
\end{algorithmic}
\end{algorithm} 

An alternative to distributing the data, say when several machines are not available, is to perform the greedy selection \emph{stochastically}. The $\stogreedy$ algorithm for submodular functions was introduced by~\citet{Mirzasoleiman:2015th}. At any given iteration $i\in [k]$, instead of performing a function evaluation for each of the remaining $(d - i)$ candidates, a subset of a fixed size $C=\lceil \frac{d \log \nicefrac{1}{\delta}}{k} \rceil$ (where $\delta$ is a pre-specified hyperparameter) is uniformly sampled from the available $(d - i)$ choices using the subroutine \textsc{Subsample}, and the function evaluation is made on those subsampled choices as if they were the only available candidates. This speeds up the greedy algorithm to $O(Ck)$ function evaluations. The algorithm is presented in Algorithm~\ref{algo:stochasticGreedy}, and its approximation bounds are discussed in Section~\ref{sec:stochastic}.

\begin{algorithm}[h]
\caption{\stogreedy($\sfS$, $k$, $\delta$)}
\label{algo:stochasticGreedy}
\begin{algorithmic}[1]
\STATE Input: sparsity $k$, available choices $\sfS$, subsampling parameter $\delta$
\STATE $\sfA_0 = \emptyset$
\FOR {$ i \in 0 \ldots (k-1)$ }
\STATE $\sfS_\delta \leftarrow $ \textsc{Subsample}($ \sfS \backslash \sfA_{i}$, $\delta$, k)
\STATE $s$ = $\argmax_{j \in \sfS_\delta} f(\sfA_{i} \cup \{j\}) - f(\sfA_i)$
\STATE $\sfA_{i+1} = \sfA_{i} \cup \{s\}$
\ENDFOR
\STATE return $\sfA_k$
\end{algorithmic}
\end{algorithm}

Finally, we discuss a property of the greedy algorithm, which is fundamental to the analysis of the $\distgreedy$ algorithm. The greedy algorithm belongs to a larger class of algorithms called $1$-nice algorithms~\citep{Mirrokni:2015kk}. The following result allows us to remove or add unselected items from the choice set that is accessible to the algorithm. 

\begin{lemma}~\citep{Mirrokni:2015kk}
\label{lem:betanice}
Say $|\sfS| > k$, and let $ \greedy(\sfS,k) \subset \sfS$ be the $k$-sized set returned by Algorithm~\ref{algo:greedy}. For any $x \notin \greedy(\sfS,k)$, $\greedy(\sfS\backslash \{x\},k)  = \greedy(\sfS,k)    $.
\end{lemma}

Note that Lemma~\ref{lem:betanice} is a property of the algorithm, and is independent of the function itself. Prior works~\citep{Mirrokni:2015kk},~\citep {Barbosa:2015vq} have exploited this property in conjunction with properties of submodular functions to obtain approximation bounds for the distributed algorithms. Our work extends these results to weakly submodular functions. As such, it is easy to see that our results are easily extensible to other \emph{nice} algorithms -- including distributed OMP and distributed stochastic greedy -- that have closed form bounds for the respective single machine algorithm. For ease of exposition, we focus our discussion on the distributed greedy algorithm. 
\section{Distributed Greedy}
\label{sec:distributed}
In this section, we obtain approximation bounds for $\distgreedy$ (Algorithm~\ref{algo:distributedgreedy}). The algorithm returns the best out of $(l+1)$ solutions : the $l$ \emph{local} solutions (steps 2,4), and the final aggregated one (step 3). Our strategy to obtain the approximation bound for the algorithm is as follows. To obtain an overall approximation bound, we obtain individual bounds on each of the solutions in terms of the submodularity ratio (Definition~\ref{def:submodularityRatio}) and use the subadditivity ratio (Definition~\ref{def:subaddivityRatio}) to show that one of the two shall always hold. For approximation bounds on the \emph{local} solutions, we make use of the \emph{niceness} of the $\greedy$ (Lemma~\ref{lem:betanice}). The bound on the aggregated solution is more involved, since it involves tracking the split of the true solution $\sfA^\star$ across machines. The assumption of partitioning uniformly at random is vital here. This helps us lower bound the greedy gain in expectation by a probabilistic overlap with the true solution. The trick of tracking the split of the true solution across machines is similar to the one that has been used for analysis of submodular functions~\citep{Mirrokni:2015kk},~\citep {Barbosa:2015vq}, but without the explicit connection to submodularity and subadditivity ratios. As we shall see in Sections~\ref{sec:linearRegression},~\ref{sec:generalfunctions} elucidating these connections leads to novel bounds for support selection for linear regression and general concave functions. 

We next define the subadditivity ratio, which helps us generalize subadditive functions in the way similar to how submodularity ratio generalizes submodular functions. 

\begin{definition}[Subaddivity ratio]
\label{def:subaddivityRatio}
We define the subadditivity ratio for a set function $f$ w.r.t a set $\sfS$ as: 
\begin{equation*}
\nu_{\sfS} := \min_{\substack{\sfA \cup \sfB = \sfS\\ \sfA \cap \sfB = \phi}} \frac{f(\sfA) + f(\sfB)}{f(\sfS)}.
\end{equation*}
We further define the subadditivity ratio of a function for an integer $k$, $\nu_k$, which  takes a uniform bound over all sets of size $k$:
\begin{equation*}
\nu_{k} := \min_{\sfS : | \sfS| = k} \nu_{\sfS}.
\end{equation*}

\end{definition}

By definition, the function $f(\cdot) $ is subadditive iff $\nu_\sfS \geq 1, \forall \sfS \subset [d]$. Since submodularaity implies subadditivity (the converse is not always true), if the function $f(\cdot)$ is submodular, $\nu_\sfS \geq 1,  \forall \sfS \subset [d]$.  

We next present some notation and few lemmas that lead up to the main result of this section (Theorem~\ref{thm:distributedgreedy}). Let $\sfA$ be the entire set of available choices. Partition the set $\sfA$ uniformly at random into $\sfA_1,\ldots, \sfA_l $. Let $\sfG_j$ be the $k$-sized solution returned by running the greedy algorithm on $\sfA_j$ \textit{i.e.} $\sfG_j = \greedy(\sfA_j, k)$. Note that each $\sfA_j$ induces a partition onto the optimal $k$-sized solution $\sfA^*$ as follows: 
\begin{eqnarray*}
\sfS_j := \{ x \in \sfA^*: x \in \greedy (\sfA_j \cup x, k)\}, \\
\sfT_j := \{ x \in \sfA^*: x \notin \greedy(\sfA_j \cup x, k)\}.
\end{eqnarray*} 

Having defined the notation, we start by lower bounding the local solutions in terms of value of the subset of $\sfA^\star$ that is not selected as part of the respective local solution. 

\begin{lemma}
\label{lem:individualgood}
$f(\sfG_j) \geq (1 - \exp (- \gamma_{\sfG_j,k} ))  f( \sfT_j ) $.
\end{lemma}

The next lemma is used to lower the bound the value of the aggregated solution (step 4  in Algorithm~\ref{algo:distributedgreedy}) in terms of the value of the subset $\sfA^\star$ that is selected as part of the respective local solution.

\begin{lemma}
\label{lem:overallgreedyonpartitions}
$\exists j \in [l]$ s.t. $\bbE[ f(\sfG) ] \geq   \left( 1 - \frac{1}{e^{\gamma_{\sfG,k}}}\right) f(\sfS_j)$.
\end{lemma}

We are now ready to present our main result about the approximation guarantee for Algorithm~\ref{algo:distributedgreedy}. 

\begin{theorem}
\label{thm:distributedgreedy}
Let $\sfG_{dg}$ be the set returned by the distributed greedy (Algorithm~\ref{algo:distributedgreedy}). Let $\gamma = \min \{\gamma_{\sfG_i,k}, \gamma_{\sfG,k} \} $. Then, 

\begin{equation}
\bbE[f(\sfG_{dg})] \geq  \frac{\nu_k}{2}    \left(1 - \exp (- \gamma )\right) f(\sfA^*) .
\end{equation}
\end{theorem}
\begin{proof}
There are $l$ machines, each with its local greedy solution $\sfG_i, i \in [l]$. In addition, there is the aggregated solution set $\sfG$. The key idea is to show that atleast one of the $(l+1)$ solutions is \emph{good enough}. 

Say $f(\sfT_i) \geq \frac{\nu_k}{2} f(\sfA^*) $ for some $i$, then by Lemma~\ref{lem:individualgood},   $f(\sfG_i) \geq \frac{\nu_k}{2}  (1 - \exp (- \gamma_{\sfG_i,k} ))  f(\sfA^*) $.

On the other hand, say for all $i$, $f(\sfT_i ) < \frac{\nu_k}{2} f(\sfA^*)$, then for all $i$, $f(\sfS_i) > \frac{\nu_k}{2} f(\sfA^*)$. By Lemma~\ref{lem:overallgreedyonpartitions}, the result then follows. 
\end{proof}

Theorem~\ref{thm:stochasticgreedy} generalizes the approximation guarantee of $\frac{1}{2} (1-\frac{1}{e})$ obtained by~\citet{Barbosa:2015vq} for submodular functions. Their analysis uses convexity of the Lovasz extension of submodular functions, and hence can not be trivially extended to weakly submodular functions. In addition to being applicable for a larger class of functions, our result can also provide tighter bounds for specific applications or datasets even for submodular functions, since they are also applicable for submodular functions, and bounding $\nu_k$ and $\gamma$ away from 1 from domain knowledge will give tighter approximations than the generic bound of $\frac{1}{2}(1- \frac{1}{e})$.

\section{Stochastic Greedy}
\label{sec:stochastic}

For analysis of Algorithm~\ref{algo:stochasticGreedy}, we show that the subsampling parameter $\delta$ governs the tradeoff between the speedup and the loss in approximation quality \emph{vis-a-vis} the classic $\greedy$. Before formally providing the approximation bound, we present an auxillary lemma that is key to proving the new approximation bound. The following result is a generalization of Lemma 2 from~\cite{Mirzasoleiman:2015th} for weakly submodular functions.

\begin{lemma}
\label{lem:randomsubset}
Let $\sfA, \sfB \subset [n]$, with $| \sfB| \leq k$. Consider another set $\sfC$ drawn randomly from $[n]\backslash \sfA$ with $|\sfC| = \lceil \frac{n\log \nicefrac{1}{\delta}}{k} \rceil$. Then,
\begin{align*}
\bbE[ \max_{v \in \sfC} f( v \cup \sfA) - f(\sfA)] \geq \frac{(1 -\delta) \gamma_{\sfA,\sfB\backslash \sfA}}{k} (f(\sfB) - f(\sfA)).
\end{align*}
\end{lemma}
We are now ready to present our result that shows that stochastic greedy selections (Algorithm~\ref{algo:stochasticGreedy}) can be applied to weakly submodular functions with provable approximation guarantees. All the proofs missing from the main text are presented in the supplement.

\begin{theorem}
\label{thm:stochasticgreedy}
Let $\sfA^\star$ be the optimum set of size $k$, and $\sfA_i = \{a_1, a_2, \ldots, a_i\}, i \leq k$ be the set built by $\stogreedy$ at step $i$.  Then, 
$ \bbE [ f(\sfA_k)]  \geq \left( 1 - \frac{1}{e^{\gamma_{\sfA_k,k}}} - \delta \right) f(\sfA^*)$.
\end{theorem}
\begin{proof}
Define $g_i := f( \sfA_{i}) - f(\sfA_{i-1})$.
Using Lemma~\ref{lem:randomsubset} with $\sfB = \sfA^\star$ and $\gamma_{\sfA_{i-1},\sfB\backslash \sfA}\geq \gamma_{\sfA_k, k} $, we get at the 
%$i^{\text{th}}$  $i$ step,
$i$-th step,
\begin{align}\nonumber
\bbE[  g_i | \sfA_{i-1} = \sfA ] &\geq&  \frac{ ( 1 - \delta ) \gamma_{\sfA_k, k} }{k} (f(\sfB) - f(\sfA))\\
 & \geq & \frac{ ( 1 - \delta ) \gamma_{\sfA_k, k} }{k} (f(\sfA^\star) - f(\sfA)).\label{tempeqn:stochasticgreedy}
\end{align}
Define $h_{i-1} :=  \bbE[  f( \sfA^\star) - f(\sfA_{i-1})]$, $C:=\frac{ ( 1 - \delta ) \gamma_{\sfA_k,k} }{k}$.   
%\begin{align}
%\label{tempeqn:stochasticgreedy}
%\bbE[  f( \sfA_{i}) - f(\sfA_{i-1}) | \sfA_{i-1} = \sfA ] \geq  \frac{ ( 1 - \delta ) \gamma_{\sfA^\star,k} }{k} (f(\sfA^*) - f(\sfA)).
%\end{align}
%Define $g_i := \bbE[  f( \sfA_{i}) - f(\sfA_{i-1})], h_{i-1} :=  \bbE[  f( \sfA^\star) - f(\sfA_{i-1})]$, $C=\frac{ ( 1 - \delta ) \gamma_{\sfA^\star,k} }{k}$.   
Note that $\bbE[g_i] = h_{i-1} - h_{i}$. Taking expectation on both sides over $\sfA_i$, ~\eqref{tempeqn:stochasticgreedy} becomes
 
 \begin{align*}
h_{i}\leq  (1 - C) h_{i-1}  \leq   (1 - C)^i  h_{0} .
 \end{align*}

Using 
$h_k = \bbE[  f( \sfA^\star) - f(\sfA_k)]$ and 
$h_0 =   f( \sfA^\star) $ above, along with the fact that $1 + a \delta \geq a^\delta $ for $0 \leq \delta \leq 1$,

\begin{eqnarray*}
\bbE [ f(\sfA_k)]  &\geq & \left( 1 - \left(1 - \frac{ ( 1 - \delta ) \gamma_{\sfA_k,k} }{k} \right)^k \right) f(\sfA^*)\\
&\geq & \left( 1 - \exp \left( - \gamma_{\sfA_k,k}\left(1 - \delta\right)  \right) \right)f(\sfA^*) \\
&\geq & \left( 1 - \frac{1}{e^{\gamma_{\sfA_k,k}}} - \delta \right) f(\sfA^*) . \qedhere
\end{eqnarray*}

\end{proof}

Note that $\delta$ is the tradeoff hyperparameter between the speedup achieved by subsampling and the corresponding approximation guarantee. A larger value of $\delta$ means the algorithm is faster with weaker guarantees and vice versa. As $\delta \rightarrow 0$, we tend towards the bound $(1-\frac{1}{e^{\gamma_{\sfA_k,k}}})$ which recovers the bound for weakly submodular functions obtained by~\citet{Kempe:2011ue} for the classic greedy selections (Algorithm~\ref{algo:greedy}), and also recovers the well known bound of $(1-\frac{1}{e})$ for submodular functions. 
\section{Large Scale Sparse Linear Regression}
\label{sec:linearRegression}
In this section, we derive novel bounds for greedy support selections for linear regression using both Algorithms~\ref{algo:distributedgreedy},~\ref{algo:stochasticGreedy}. Recall that $\bX \in \bbR^{n \times d}$ is the \emph{feature matrix}, with $n$ samples and $d$ features, and $\by \in \bbR^n$ is the vector of $n$ \emph{responses}. We assume, without loss of generality, that the columns of the matrix $\bX$ are normalized to $1$, and the response vector is also normalized to have norm $1$. Let $\bC \in \bbR^{d \times d}$ be the covariance matrix. 

We know from standard linear algebra, that for a fixed set of columns $\bX_\sfS$, where $\sfS \subset [d]$ is the index into columns of $\bX$, $\bbeta_\sfS^\star =  (\bX_\sfS^\top \bX_\sfS)^{-1} \bX_\sfS^\top \by$. Minimizing the error in~\eqref{eqn:linearRegression} is thus equivalent to maximizing the following set function (modulo a constant $\|\by\|_2^2$):

\begin{equation}
\label{eq:linearRegressionf}
f(\sfS) := \|\bP_\sfS \by \|_2^2, 
\end{equation}

where $\bP_\sfS:= \bX_\sfS (\bX_\sfS^\top \bX_\sfS)^{-1} \bX_\sfS^\top$ is the projection matrix for orthogonal projection onto the span of columns of $\bX_\sfS$. $f(\cdot)$ as defined above is also the $R^2$ statistic for the linear regression problem~\eqref{eqn:linearRegression}. The respective combinatorial maximization is: 

\begin{equation}
\label{eq:maxfLinearRegression}
\max_{|\sfS|\leq k } f(\sfS).
\end{equation}

The function defined in ~\eqref{eq:maxfLinearRegression} is \emph{not} submodular. However, submodularity is not required for giving guarantees for greedy forward selection. A bounded submodularity ratio is a weaker condition that is sufficient to provide such approximation guarantees. ~\citet{Kempe:2011ue} analyzed the greedy algorithm for ~\eqref{eqn:linearRegression}, and showed that this function was weakly submodular. Our goal is to maximize the $R^2$ statistic for linear regression using Algorithms~\ref{algo:distributedgreedy},~\ref{algo:stochasticGreedy}.

For a positive semidefinite matrix $\bC$, $\lambda_{\text{max}}(\bC,k)$ and $\lambda_{\text{min}}(\bC,k)$ be the largest and smallest $k$-sparse eigenvalue of $\bC$. We make use of the following result from~\citet{Kempe:2011ue}:

\begin{lemma}
\label{lem:DasKempe}
For the $R^2$ statistic~\eqref{eq:linearRegressionf}, $\gamma_{\sfS, k} \geq \lambda_{\text{min}} (\bC, \, k + |\sfS|)$.
\end{lemma}

Recall that we need to only bound the submodularity ratio $\gamma_{\sfS_g, k}$ over the selected greedy set to obtain the approximation bounds (See Theorems~\ref{thm:distributedgreedy} and~\ref{thm:stochasticgreedy}). Lemma~\ref{lem:DasKempe} provides a data dependent union bound for the submodularity ratio in terms of sparse eigenvalue of the covariance matrix. We now provide the corresponding approximation bounds for Algorithm~\ref{algo:stochasticGreedy} next.

\begin{corollary}
\label{cor:linearRegression}
Let $\sfA^\star$ be the optimal support set for sparsity constrained maximization of the $R^2$ statistic~\eqref{eq:maxfLinearRegression}. Let $\sfA_{sg}$ be the solution returned by $\stogreedy$ ($\sfS, k, \delta$). Then, 
\begin{equation*}
\bbE [ f(\sfA_{sg})]  \geq \left( 1 - \frac{1}{e^{\lambda_{\text{min}} (C, \, 2k)}} - \delta \right) f(\sfA^*).
\end{equation*}
\end{corollary}

To obtain bounds for Algorithm~\ref{algo:distributedgreedy}, we also need to bound the subadditivity ratio (recall Definition~\ref{def:subaddivityRatio}).
\begin{lemma}
\label{lem:subdadditivityRegression}
For the maximization of the $R^2$~\eqref{eq:linearRegressionf}, $\nu_\sfS \geq \frac{\lambda_{\min}(\bC_\sfS)}{ \lambda_{\max(\bC_\sfS)}  }$, where $\bC_\sfS$ is the submatrix of $\bC$ with rows and columns indexed by $\sfS$.
\end{lemma}

We can now provide the bounds for greedy support selection using Algorithm~\ref{algo:distributedgreedy} for the linear regression problem~\eqref{eq:linearRegressionf}.

\begin{theorem}
Let $\sfA_{dg}$ be the solution returned by the $\distgreedy$ algorithm, and let $\sfA^\star$ be the optimal solution for the sparsity constrained maximization of $R^2$~\eqref{eq:maxfLinearRegression}. Then, 
\begin{align*}
f(\sfA_{dg}) & \geq  \frac{1}{2}  \frac{\lambda_{\min}(\bC_{\sfA^\star})}{ \lambda_{\max}(\bC_{\sfA^\star})  }  \left(1 - \exp (- \lambda_{\text{min}} (\bC, \, 2k) )\right) f(\sfA^*) \\
&\geq \frac{1}{2}  \frac{\lambda_{\min}(\bC,k)}{ \lambda_{\max}(\bC,k)  }  \left(1 - \exp (- \lambda_{\text{min}} (\bC, \, 2k) )\right) f(\sfA^*).
\end{align*}
\end{theorem}
\begin{proof}
Follows from Lemma~\ref{lem:subdadditivityRegression} and by a uniform bound on $\gamma$ in Theorem~\ref{thm:distributedgreedy} as $\gamma \geq \lambda_{\text{min}}(\bC, 2k)$ (from Lemma~\ref{lem:DasKempe}).
\end{proof}

\section{Support Selection for general functions}
\label{sec:generalfunctions}
In this section, we leverage recent results from connections of convexity to submodularity to provide support selection bounds for Algorithms~\ref{algo:distributedgreedy},~\ref{algo:stochasticGreedy} for general concave functions. 

The sparsity constraint problem with a given $k\leq d$ for a concave function $g: \bbR^d \rightarrow \bbR$ is:

\begin{equation}
\label{eqn:supportSelectionf}
\max_{\|\bx\|_0 \leq k } g(\bx).
\end{equation}
Similar to the developments in Section~\ref{sec:linearRegression}, we can define an associated set function as: 
\begin{equation}
\label{eqn:generalSetFunction}
f(\sfS) := \max_{\text{supp}(\bx) \subset \sfS } g(\bx) - g(\mathbf{0}).
\end{equation}

We recall that the submodular guarantee of $(1-\frac{1}{e})$ is for \emph{normalized} submodular functions. To extend the notion of normalization to general support selection, we subtract $g(\mathbf{0})$. 

To bound the submodularity ratio for $f(\cdot)$ in~\eqref{eqn:generalSetFunction}, the concept of strong concavity and smoothness is required. 
%A function $g: \bbR^d \rightarrow \bbR$ is $m_\Omega$ restricted strongly concave and $L_\Omega$ restricted strongly smooth over a subdomain $\Omega \subset \bbR^d$,
%\begin{equation*}
%\frac{m_\Omega}{2} {\| \by -\bx\|}_2^2 \geq g(\by) - g(\bx) - \langle \nabla g(\bx) , \by - \bx \rangle \geq  - \frac{M_\Omega}{2}  {\| \by -\bx\|}_2^2 
%\end{equation*}
For a function $g: \bbR^d \rightarrow \bbR$, define $\mathcal{D}_g(\bx,\by)~:=~g(\by)~-~g(\bx)~-~\langle \nabla g(\bx) , \by - \bx \rangle$.  We say $g(\cdot)$ is $m_\Omega$ Restricted Strongly Concave (RSC) and $L_\Omega$ Restricted Strongly Smooth (RSM) over a subdomain $\Omega \subset \bbR^d$ if 
\begin{equation*}
-\frac{m_\Omega}{2} {\| \by -\bx\|}_2^2 \geq \mathcal{D}_g(\bx,\by) \geq  - \frac{L_\Omega}{2}  {\| \by -\bx\|}_2^2  .
\end{equation*}

  We make use of the following result that lower bounds the submodularity ratio for $f(\cdot)$:
\begin{lemma}[\citet{Elenberg:2016tq}]
\label{lem:elenberg}
If the given function $g(\cdot)$ is $m$-strongly concave on all $| \sfS |  + k $ sparse supports and $L$-smooth over all $| \sfS| + 1$ sparse supports, 
\begin{equation*}
\gamma_{\sfS,k} \geq \frac{m}{L}.
\end{equation*}
\end{lemma}

\begin{corollary}
Say the function $ g(\bx):\bbR^d \rightarrow \bbR$ satisfies the assumptions of Lemma~\ref{lem:elenberg}. Let $\sfA^\star$ be the optimal support set that maximizes $g(\cdot)$ under the sparsity constraint~\eqref{eqn:supportSelectionf}. Let $\sfA_{sg}$ be the solution set returned by $\stogreedy$. Then, 
\begin{equation*}
\bbE [ f(\sfA_{sg})]  \geq \left( 1 - \frac{1}{e^{\nicefrac{m}{L}}} - \delta \right) f(\sfA^*).
\end{equation*}

\end{corollary}

\subsection{RSC implies Weak Subadditivity}
In this section, we establish a lower bound on the subadditivity ratio in terms of only the restricted strong concavity (RSC) and smoothness (RSM) constants. This is analogous to lower bounding the submodularity ratio by~\citet{Elenberg:2016tq}.

\begin{theorem}
\label{thm:RSCsubadd}
Say the function $ g(\bx):\bbR^d \rightarrow \bbR$ is $m$ strongly concave and $L$-smooth over all $\bx$ supported on $\sfS$. Then, the subadditivity ratio can be lower bounded as: 
\begin{equation}
\nu_\sfS \geq  \frac{m}{L}.
\end{equation}
\end{theorem}

\begin{proof}
To prove Theorem~\ref{thm:RSCsubadd}, we make use of the following two results. Recall that for a set $\sfS$ and vector $\bu$, $\bu_\sfS$ denotes the subvector of $\bu$ supported on $\sfS$. 

\begin{lemma}
\label{lem:rsclem1}
For a support set $\sfS \subset [d]$, $f(\sfS) \geq \frac{1}{2 L} \| \nabla g(\mathbf{0})_\sfS\|_F^2$.
\end{lemma}

\begin{lemma}
\label{lem:rsclem2}
For any support set $\sfS \subset [d]$,  $f(\sfS) \leq \frac{1}{2 m} \|\nabla g(\mathbf{0})_\sfS \|_F^2$.
\end{lemma}

Let $\sfA, \sfB$ be a partition of a given support set $\sfS \subset [d]$ i.e $\sfA \cup \sfB = \sfS$, $\sfA \cap \sfB = \{\}$. 

We can use Lemma~\ref{lem:rsclem1} to lower bound the numerator of the subadditivity ratio as follows:

 \begin{eqnarray}\nonumber
 f(\sfA) + f(\sfB) &\geq & \frac{1}{2 L} \left( \| \nabla g(\mathbf{0})_\sfA\|_F^2 + \| \nabla g(\mathbf{0})_\sfB\|_F^2  \right) \\ \label{eqnproof:rscthm1}
 &=&  \frac{1}{2 L}  \| \nabla g(\mathbf{0})_\sfS\|_F^2.
  \end{eqnarray} 

Combining~\eqref{eqnproof:rscthm1} with Lemma~\ref{lem:rsclem2}, we get,
\vspace{-2mm}
\begin{eqnarray*}
\nu_\sfS= \frac{f(\sfA) + f(\sfB)}{f(\sfS) } \geq  \frac{m}{L}.\qedhere
\end{eqnarray*}
\end{proof}

Since we have a bound for the subadditivity ratio for general strongly concave and smooth functions, we can now provide a novel approximation guarantee for support selection by $\distgreedy$.

\begin{corollary}
\label{cor:generalfunctions_distributed}
Say the function $ g(\bx):\bbR^d \rightarrow \bbR$ is $m$-strongly concave over all $2k$ sparse support sets and $L$-smooth over all $k+1$ sparse support sets. Let $\sfA^\star$ be the optimal support set that maximizes the sparsity constrained $g(\cdot)$~\eqref{eqn:supportSelectionf}. Let $\sfA_{dg}$ be the solution set returned by $\distgreedy$. Then, 
\begin{equation*}
%f(\sfA_{dg}) \geq  \frac{m}{2L}    \left(1 - \exp (- \nicefrac{m}{L} )\right) f(\sfA^*) .
\bbE[f(\sfA_{dg})] \geq  \frac{m}{2L}    \left(1 - e^{- \nicefrac{m}{L} }\right) f(\sfA^*) .
\end{equation*}
\end{corollary}

To the best of our knowledge, the bounds obtained in Corollary~\ref{cor:generalfunctions_distributed} are the first for a distributed algorithm for support selection for general functions. Note that we have taken a uniform bound for restricted strong concavity and smoothness to be over all $k$ sized support sets, though it is only required to be over the optimal support set.

\section{Experiments}

\begin{figure}[ht]
\centering
\begin{subfigure}{0.45\columnwidth}
\includegraphics[width=\textwidth]{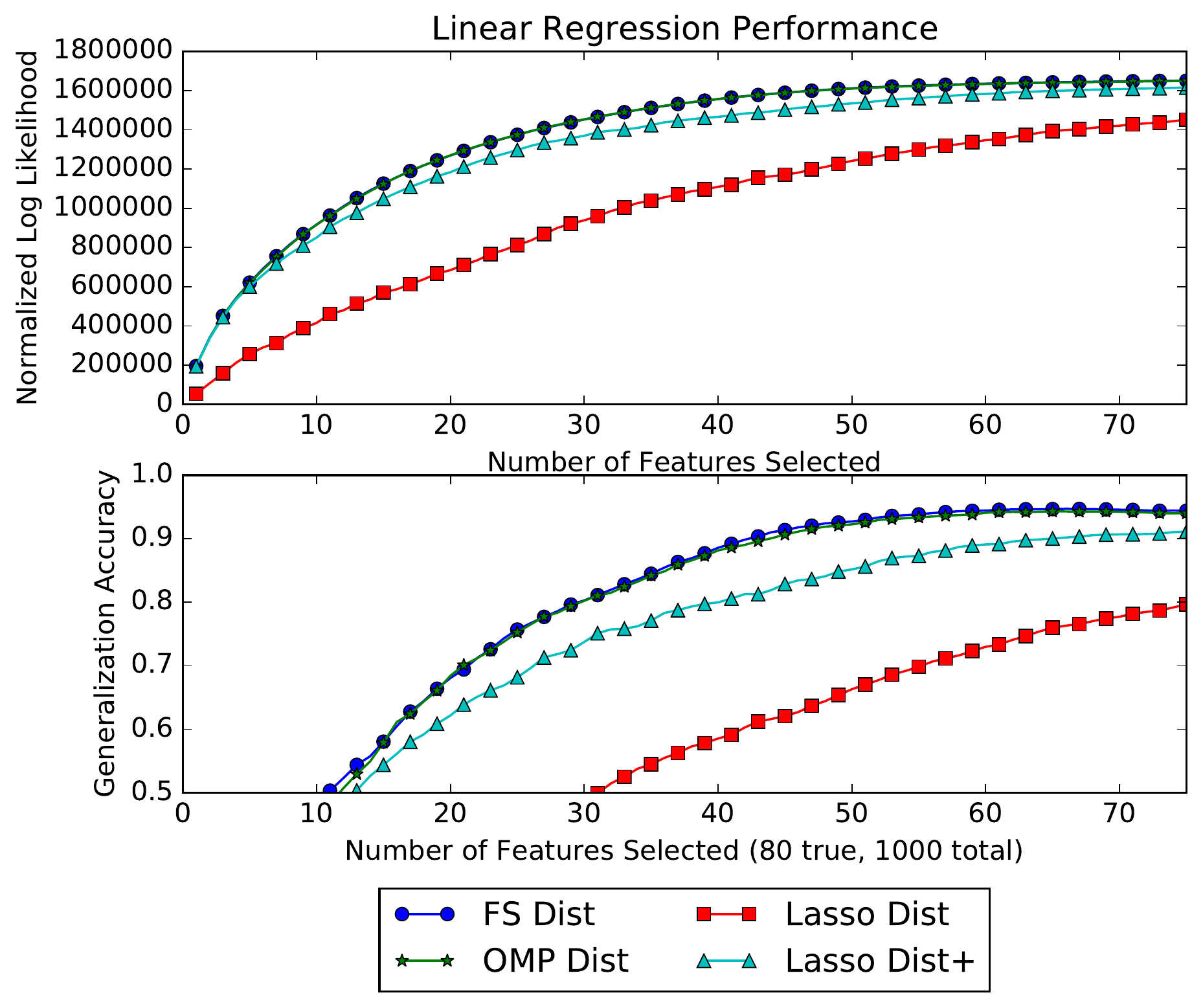}
\caption{}\label{fig:synthetic2ll}
\end{subfigure}
\begin{subfigure}{0.45\columnwidth}
\includegraphics[width=\textwidth]{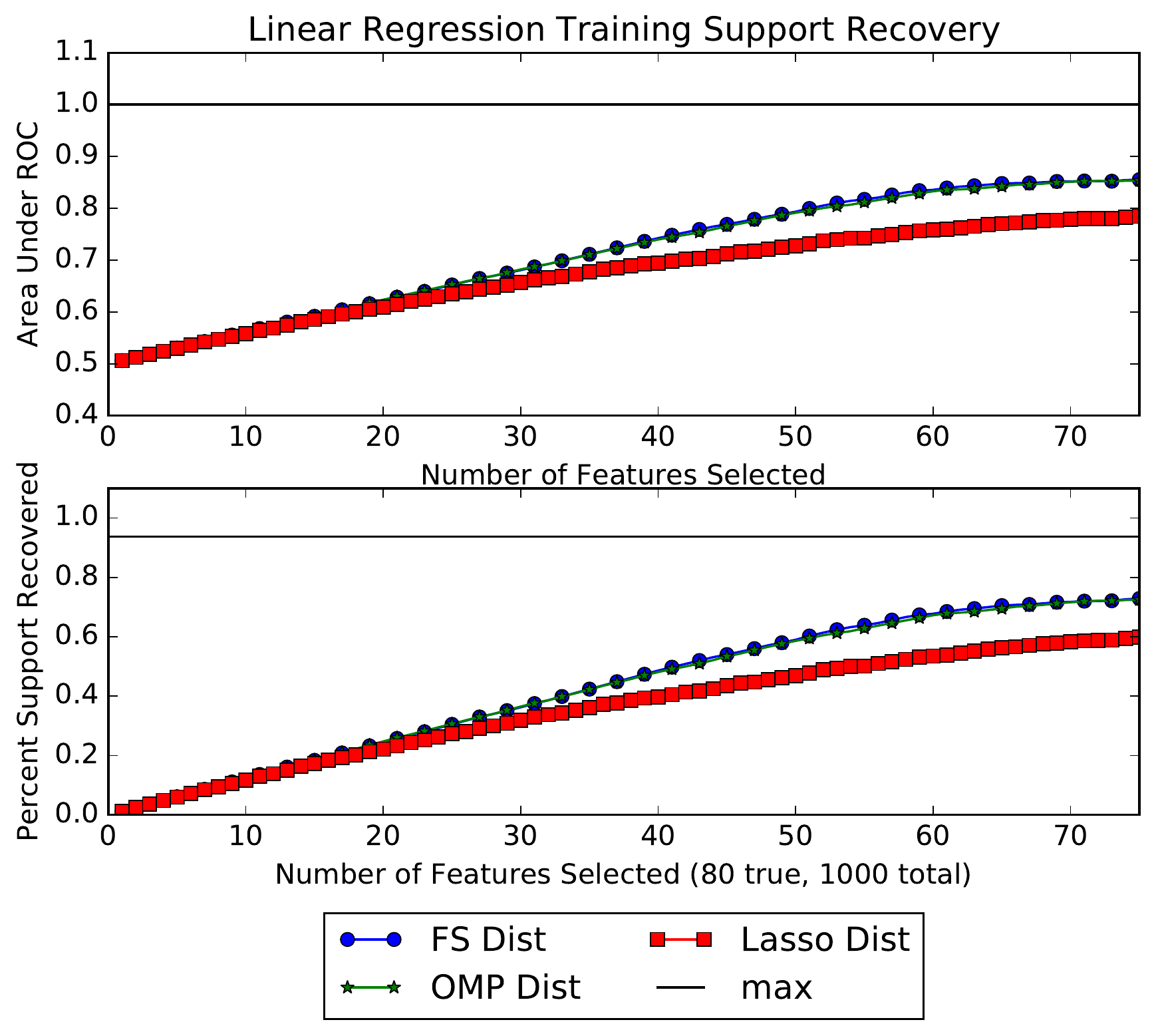}
\caption{}\label{fig:synthetic2supp}
\end{subfigure}
\caption{\label{fig:synthetic2} Distributed linear regression, $l=10$ partitions, $n=800$ training and test samples, $\alpha=0.5$. Results averaged over $10$ iterations. Both greedy algorithms outperform $\ell_1$ regularization.}
\end{figure}

\begin{figure}[ht]
\centering
\includegraphics[width=0.7\columnwidth,clip=true,trim= 0 0 0 0.29in]{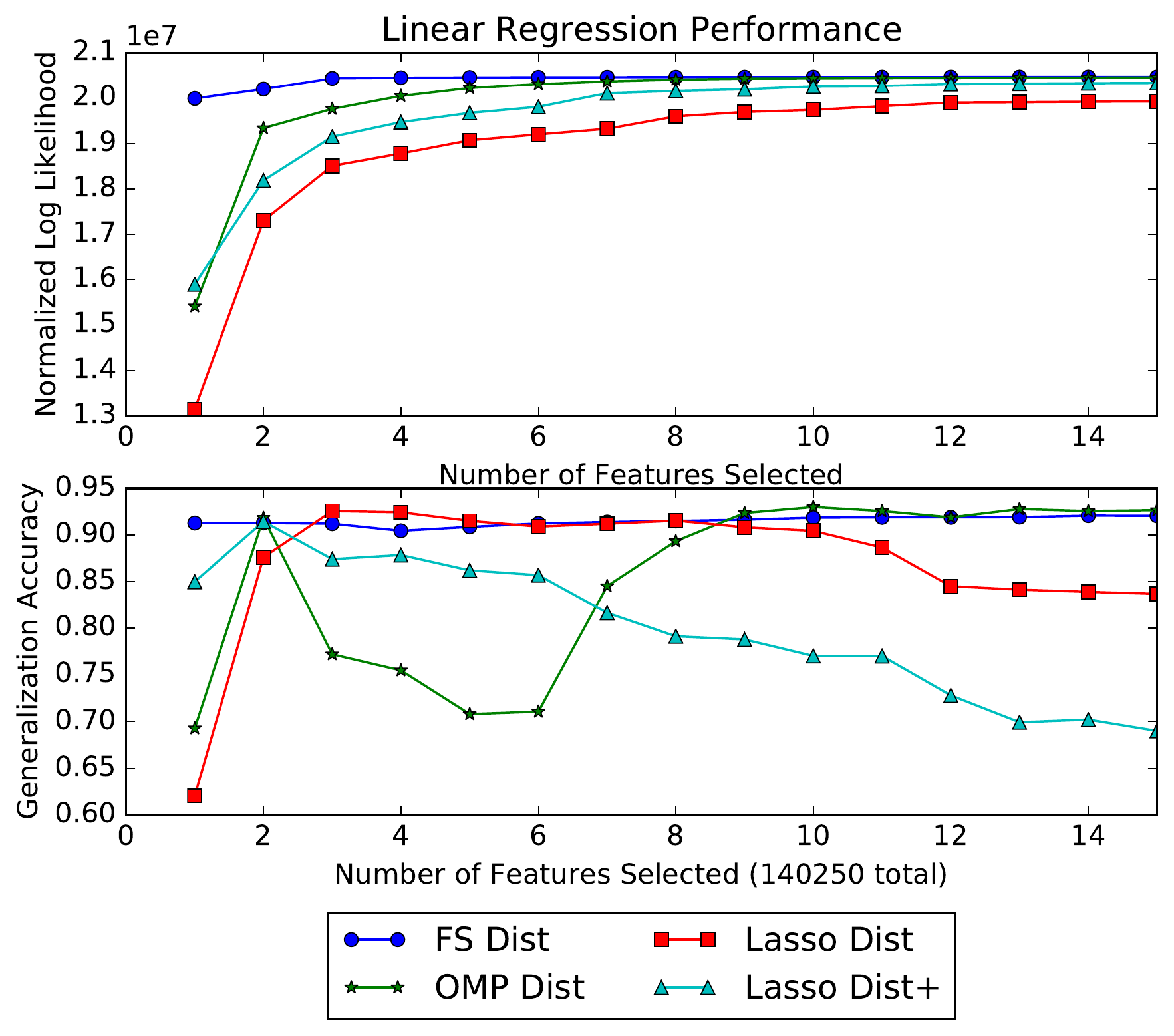}
\caption{Distributed linear regression, Electricity dataset. }\label{fig:elec1}
\end{figure}

\subsection{Distributed Linear Regression}
We consider sparse linear regression in a distributed setting. We generate a $100$-sparse regression vector is generated by selecting random nonzero entries of $\bbeta$,
\begin{align*}
\bbeta_{s} = (-1)^{\operatorname{Bern}(\nicefrac{1}{2})} \times \left(5\sqrt{\frac{\log d}{n}} + \delta_{s} \right),  
\end{align*}
 where $\delta_{s}$ is a standard i.i.d. Gaussian. Measurements $\by$ are taken according to $\by = \bX\bbeta + \bz$, where $\forall  i \in [n], z_i$ is i.i.d. Gaussian with variance set to be $0.01 \normbs{\bX \bbeta}$. Each row of the design matrix $\bX$ is generated by an autoregressive process,
\begin{align*}
X_{n,t+1} = \sqrt{1 - \alpha^2}X_{n,t} + \epsilon_{n,t},
\end{align*}
where $\epsilon_{n,t}$ is i.i.d. Gaussian with variance $\alpha^2=0.25$. 
%Intuitively, lower $\alpha$ means more correlation among features. 
We take $n=800$ for the number of both training and test measurements. Results are averaged over $10$ iterations, each with a different partition $\{\sfA_j\}$ of the $1000$ features.

We evaluate four variants of $\distgreedy$. The two greedy algorithms are $\greedy$ Forward Selection (FS) and Orthogonal Matching Pursuit (OMP). Lasso sweeps an $\ell_1$ regularization parameter $\lambda$ using LARS \citep{EfronLARS}. This produces nested subsets of features corresponding to a sequence of thresholds for which the support size increases by $1$. Lasso uses this threshold, while Lasso+ fits an unregularized linear regression on the support set selected by Lasso.

Figure~\ref{fig:synthetic2} shows the performance of all algorithms on the following metrics: log likelihood (normalized with respect to a null model), generalization to new test measurements from the same true support parameter, area under ROC, and percentage of the true support recovered for $l=10$.

%template has vspaces??

\begin{figure}[h]
\centering
%\vspace{.3in}
%\centerline{\fbox{This figure intentionally left non-blank}}
\includegraphics[width=0.7\columnwidth]{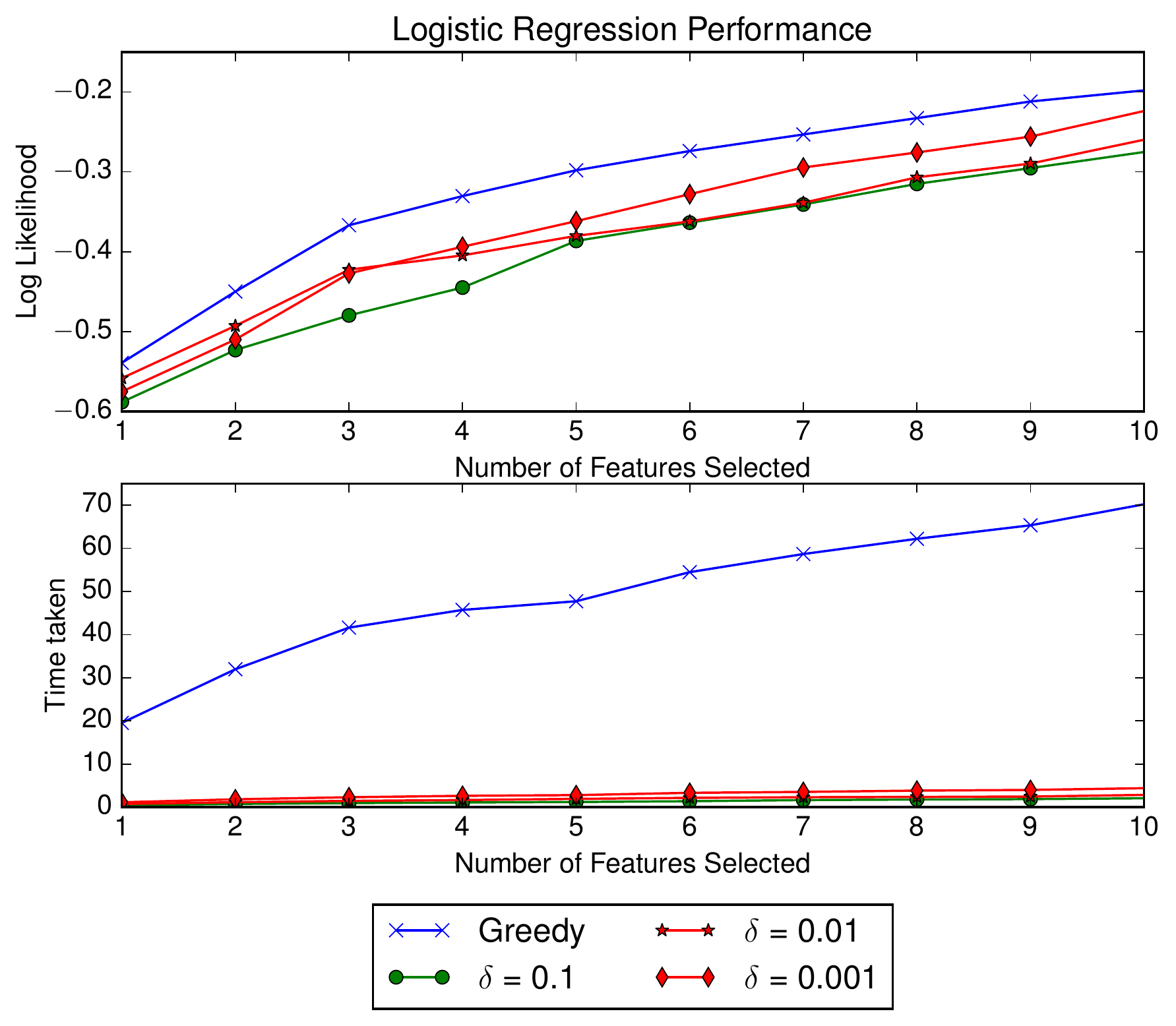}
\caption{\label{fig:lazy} Trade off in time vs log likelihood for various values of $\delta$-Stochastic Greedy as opposed to Greedy Forward Selection for logistic regression on \texttt{gisette} data~\citep{UCI}. Results averaged over $10$ iterations.}
\end{figure}

Next, we run a similar experiment on a large, real-world dataset. We sample $d=140,\!250$ time series measurements across $n=370$ customers from the \texttt{ElectricityLoadDiagrams} time series dataset \citep{UCI}. We consider the supervised learning experiment of predicting the electrical load at the \emph{next} time $140,\!251$. We use half of the customers for training and the rest for testing. Figure \ref{fig:elec1} shows performance of the same algorithms with data distributed across $l=50$ partitions to select the top $k=15$ features. We see that distributed Forward Selection produces both largest likelihood and highest generalization score. OMP has second largest likelihood, but its generalization varies widely for different values of $k$. This is likely due to the random placement of predictive features across a large number of partitions.

\subsection{Stochastic Sparse Logistic Regression}
In this section we demonstrate the applicability of Algorithm~\ref{algo:stochasticGreedy} for greedy support selection for sparse logistic regression. Note that the respective set function~\eqref{eqn:generalSetFunction} when $g(\cdot)$ is the log likelihood for logistic regression is not submodular. As such one would not typically apply the $\stogreedy$ algorithm for sparse logistic regression. However, the guarantees obtained in Section~\ref{sec:generalfunctions} suggest good practical performance which is indeed demonstrated in Figure~\ref{fig:lazy}. For the experiment we use the \texttt{gisette} dataset obtained from the UCI website~\citep{UCI}. The dataset is of a handwritten digit recognition problem to separate out digits `$4$' and `$9$'. It has 13500 instances, and 5000 features. Figure~\ref{fig:lazy} illustrates depicts the tradeoff between the time taken to learn the model and the respective training log likelihood for different values of $\delta$ as used in Algorithm~\ref{algo:stochasticGreedy}. As shown, we obtain tremendous speed up with relatively little loss in the log likelihood value even for reasonably large $\delta$ values. 

\section{Conclusion}
We provided novel bounds for two greedy algorithm variants for maximizing weakly submodular functions with applications to linear regression and general concave functions. Our research opens questions on what other known algorithms with provable guarantees for submodular functions can be extended to weakly submodular functions.

\section*{Acknowledgement}
Research supported by NSF Grants CCF 1344179, 1344364, 1407278, 1422549, IIS 1421729, and ARO YIP W911NF-14-1-0258.
\newpage\pagebreak

\bibliographystyle{IEEEtranN}
\bibliography{biblio}

\newpage\pagebreak

\appendix
\section{Proofs}
In this section, we provide detailed proofs of all the results used in this manuscript. For lemma and theorem statements repeated from the main text, we add an apostrophe to indicate that it is not a new lemma/theorem being introduced. 

We make use of the following result that provides an approximation bound for greedy selections for weakly submodular functions. 

\begin{lemma}\label{lem:DasKempeGreedyGuarantee}\citep{Kempe:2011ue})  Let $\sfS^\star$ be the optimal $k$-sized set that maximizes $f(\cdot)$ under the $k$-sparsity constraint (see~\eqref{eq:maxf}). Let $\sfS_G$ be the set returned by greedy forward selection (Algorithm~\ref{algo:greedy}), then 
\begin{equation*}
f(\sfS_G) \geq (1 - \exp( - \gamma_{\sfS_G,k}) f(\sfS^\star).
\end{equation*}
\end{lemma}

\subsection{Distributed Greedy}

\begin{replemma}{lem:individualgood}
$f(\sfG_j) \geq (1 - \exp (- \gamma_{\sfG_j,k} ))  f( \sfT_j ) $.
\end{replemma}
\begin{proof}
From Lemma~\ref{lem:betanice}, we know that running greedy on $\sfA_j \cup \sfT_j $ instead of $\sfA_j$ will still return the set $\sfG_j$, 
\begin{eqnarray*}
f(\sfG_j) &\overset{Lemma~\ref{lem:DasKempeGreedyGuarantee}}{\geq} & (1 - \exp (- \gamma_{\sfG_j,k} )) \max_{ \substack{ | \sfS | \leq k \\ \sfS \subset \sfA_j \cup \sfT_j } }  f(\sfS) \\
 &\geq &  (1 - \exp (- \gamma_{\sfG_j,k} ))  f( \sfT_j ) .
\end{eqnarray*}
\end{proof}

For proving Lemma~\ref{lem:overallgreedyonpartitions}, we require another auxillary result. 

\begin{lemma}
\label{lem:probabilityOfPartitoning}
For any $x \in  \sfA, \bbP (x \in \cup_j  \sfG_j) = \frac{1}{l} \sum_j \bbP(x \in \sfS_j)$.
\end{lemma}
\begin{proof}
We have \begin{align*}
&\bbP (x \in \cup_j  \sfG_j) \\
 &=  \sum_j \bbP ( x \in \sfA_i \cap x \in \greedy(\sfA_i,k)) \\
&= \sum_j \bbP(x \in \sfA_i) \bbP(x \in \greedy(\sfA_i,k) | x \in \sfA_i)\\
&= \sum_j \bbP(x \in \sfA_i) \bbP(x \in \sfS_i)\\
& = \frac{1}{l}\bbP(x \in \sfS_i).
\end{align*}

\end{proof}

We now prove Lemma~\ref{lem:overallgreedyonpartitions}. 

\begin{replemma}{lem:overallgreedyonpartitions}
$\exists j \in [l]$, $\bbE[ f(\sfG) ] \geq   \left( 1 - \frac{1}{e^{\gamma_{\sfG,k}}}\right) f(\sfS_j)$.
\end{replemma}

\begin{proof}

For $i \in [k]$, let $\sfB_i : \greedy( \cup_j \sfG_j, i)$, so that $\sfB_k = \sfG$ in step 3 of Algorithm~\ref{algo:distributedgreedy}.  Then,

\begin{align}\nonumber
&\bbE[ f(\sfB_{i+1}) - f(\sfB_i) ] \\\nonumber
&\geq \frac{1}{k} \sum_{x \in \sfA^*} \bbP (x \in \cup_j  \sfG_j) \bbE\left[    f(\sfB_i \cup x) - f(\sfB_i)  \right]\\\label{tempeqn:stepLemmaoverallgreedy}
 & \overset{Lemma~\ref{lem:probabilityOfPartitoning}}{=} \frac{1}{kl} \sum_{x \in \sfA^*} \left( \sum_{j=1}^l \bbP(x \in \sfS_i)\right) \bbE(f(\sfB_i \cup x) - f(\sfB_i))\\\nonumber
& =  \frac{1}{kl} \sum_{j=1}^l \sum_{x \in \sfS_j}  \bbE(f(\sfB_i \cup x) - f(\sfB_i))\\\nonumber
& = \frac{1}{kl} \sum_{j=1}^l \gamma_{\sfB_i, \sfS_j\backslash \sfB_i} \bbE(f(\sfB_i \cup \sfS_j) - f(\sfB_i)) \\ \nonumber
& \geq \frac{1}{kl} \sum_{j=1}^l \gamma_{\sfB_i, \sfS_j\backslash \sfB_i}\bbE (f( \sfS_j) - f(\sfB_i))\\\nonumber
%&  \geq \frac{1}{kl} \sum_{j=1}^l \gamma_{\sfG, k} \bbE(f( \sfS_j) - f(\sfB_i))\\\nonumber
%&    = \frac{1}{kl} \sum_{j=1}^l \gamma_{\sfG, k} \bbE(f( \sfS_j) - f(\sfB_i)) \\ \nonumber
 & \geq \frac{\gamma_{\sfB_i, k}}{k} \min_{j}  \bbE(f( \sfS_j) - f(\sfB_i)).
\end{align}

Using $\gamma_{\sfB_i, k} \geq \gamma_{\sfG, k}$, and proceeding now as in the proof of Theorem~\ref{thm:stochasticgreedy}, we get the desired result.

\end{proof}

\subsection{Stochastic Greedy}
\begin{replemma}{lem:randomsubset}
Let $\sfA, \sfB \subset [n]$, with $| \sfB| \leq k$. Consider another set $\sfC$ drawn randomly from $[n]\backslash \sfA$ with $|\sfC| = \lceil \frac{n\log \nicefrac{1}{\delta}}{k} \rceil$. Then 
\begin{align*}
\bbE[ \max_{v \in \sfC} f( v \cup \sfA) - f(\sfA)] \geq \frac{(1 -\delta) \gamma_{\sfA,\sfB\backslash \sfA}}{k} (f(\sfB) - f(\sfA)).
\end{align*}
\end{replemma}
\begin{proof}
%The random sample is drawn from $[n]\backslash\sfA$. 
To relate the best possible marginal gain from $\sfC$ to the total gain of including the set $\sfB\backslash \sfA$ into $\sfA$, we must upper bound the probability of overlap between $\sfC$ and $\sfB\backslash \sfA$ as follows:

\begin{eqnarray}\nonumber
\bbP(\sfC \cap (\sfB\backslash \sfA) \neq  \emptyset) =& 1-  \left( 1 - \frac{ |\sfB\backslash \sfA|}{| [n]\backslash \sfA|} \right)^{|\sfC|} \\\nonumber
=& 1 - \left( 1 - \frac{ |\sfB\backslash \sfA|}{| [n]\backslash \sfA|} \right)^{ \lceil \frac{n\log \nicefrac{1}{\delta}}{k} \rceil}\\\nonumber
\geq& 1 - \exp{ \left( - \frac{n\log \nicefrac{1}{\delta}}{k}  \frac{ |\sfB\backslash \sfA|}{| [n]\backslash \sfA|} \right) }\\\nonumber
\geq& 1 - \exp{ \left(- \frac{ |\sfB\backslash \sfA|\log \nicefrac{1}{\delta}}{k}\right)}\\\label{eqnproof:overlap}
\geq & \left(1 - \exp \left(- \log \nicefrac{1}{\delta}\right)\right) \frac{ |\sfB\backslash \sfA|}{k} \\ \nonumber
= & ( 1 - \delta ) \frac{ |\sfB\backslash \sfA|}{k},
\end{eqnarray}

where~\eqref{eqnproof:overlap} is because $\frac{|\sfB\backslash \sfA|}{k} \leq 1$. Let $\sfS =  \sfC \cap (\sfB\backslash \sfA)$. Since $f(v \cup \sfA) - f(\sfA)$ is 
nonnegative,
%positive for $v\in\sfC$,
\begin{align*}
\nonumber
&\bbE[ \max_{v \in \sfC} f( v \cup \sfA) - f(\sfA)]    \\ \nonumber
&\geq\bbP (\sfS\neq  \emptyset)\bbE[  \max_{v \in \sfC} f( v \cup \sfA) - f(\sfA) |\sfS\neq  \emptyset]  \\\nonumber 
&\geq  ( 1 - \delta ) \frac{ |\sfB\backslash \sfA|}{k} \bbE[  \max_{v \in \sfC} f( v \cup \sfB) - f(\sfB) |\sfS\neq  \emptyset] \\\nonumber
&\geq  ( 1 - \delta ) \frac{ |\sfB\backslash \sfA|}{k} \bbE[  \max_{v \in \sfC  \cap (\sfB\backslash \sfA)} f( v \cup \sfA) - f(\sfA) |\sfS\neq  \emptyset] \\\nonumber
&\geq  ( 1 - \delta ) \frac{ |\sfB\backslash \sfA|}{k} \sum_{v \in \sfB\backslash \sfA} \frac {f(v \cup \sfA) - f(\sfA)}{ |\sfB\backslash \sfA| } \\ \nonumber
&\geq  \frac{ ( 1 - \delta ) \gamma_{\sfA,\sfB\backslash \sfA} }{k} (f(\sfB) - f(\sfA)).
\end{align*}
\end{proof}

\subsection{Linear regression}

\begin{replemma}{lem:subdadditivityRegression}
For the maximization of the $R^2$~\eqref{eq:linearRegressionf} $\nu_\sfS \geq \frac{\lambda_{\min}(\bC_\sfS)}{ \lambda_{\max(\bC_\sfS)}  }$, where $\bC_\sfS$ is the submatrix of $\bC$ with rows and columns indexed by $\sfS$.
\end{replemma}
\begin{proof}
Say $\sfA$, $\sfB$ is an arbitrary partition of $\sfS$. Consider, 

\begin{eqnarray}
\| \bP_\sfA \by \|_2^2 &= & \by^\top \bP_\sfA \by \label{tempeqn:stepprojection}\\ \nonumber
& = & \by^\top \bX_\sfA (\bX_\sfA^\top \bX_\sfA)^{-1} \bX_\sfA^\top \by\\\nonumber
& \geq & \| \bX_\sfA^\top \by\|_2^2 \lambda_{\text{min}}\left((\bX_\sfA^\top \bX_\sfA)^{-1}\right)\\\nonumber
&= & \| \bX_\sfA^\top \by\|_2^2 \frac{1}{\lambda_{\text{max}}\left(\bX_\sfA^\top \bX_\sfA\right)}\\\nonumber
&\geq &  \| \bX_\sfA^\top \by\|_2^2 \frac{1}{\lambda_{\text{max}}\left(\bX_\sfS^\top \bX_\sfS\right)}.
\end{eqnarray}

where~\eqref{tempeqn:stepprojection} results from the fact that all orthogonal projection matrices are symmetric and idempotent. Repeating a similar analysis for $\sfB$ instead of $\sfA$, we get 
\begin{eqnarray*}
\| \bP_\sfA \by \|_2^2 + \| \bP_\sfB \by \|_2^2 &\geq &  \frac{\| \bX_\sfA^\top \by\|_2^2 + \| \bX_\sfB^\top \by\|_2^2}{\lambda_{\text{max}}\left(\bX_\sfS^\top \bX_\sfS\right)} \\
&= &  \frac{\|\bX^\top_\sfS \by\|^2_2}{\lambda_{\text{max}}\left(\bX_\sfS^\top \bX_\sfS\right)}.
\end{eqnarray*}

A similar analysis also gives $\| \bP_\sfS \by \|_2^2 \leq \frac{\| \bX_\sfS \by\|_2^2}{ \lambda_{\text{min}} \left(\bX_\sfS^\top \bX_\sfS\right)}$, which gives the desired result.
\end{proof}

\subsection{RSC implies weak subadditivity}

Let $\bbeta^{(\sfS)} := \max_{\text{supp}(\bx) \in \sfS} g(\bx) $.

\begin{replemma}{lem:rsclem1}
For a support set $\sfS \subset [d]$, $f(\sfS) \geq \frac{1}{2 L} \| \nabla g(\mathbf{0})_\sfS\|_2^2$.
\end{replemma}
\begin{proof}
For any $\bv$ with support in $\sfS$,
\begin{align*}
g(\bbeta^{(\sfS)})  - g(\mathbf{0}) & \geq g(\bv) - g(\mathbf{0}) \\
&\geq \inprod{\nabla g(\mathbf{0})}{\bv} - \frac{L}{2} \| \bv\|_2^2.
\end{align*}

Using $\bv = \frac{1}{L}\nabla g(\mathbf{0})_\sfS$, we get the desired result.  

\end{proof}
\begin{replemma}{lem:rsclem2}
For any support set $\sfS$,  $f(\sfS) \leq \frac{1}{2 m} \|  \nabla g(\mathbf{0})_\sfS \|_2^2$.
\end{replemma}
\begin{proof}
By strong concavity, 
\begin{align*}
g(\bbeta^{(\sfS)}) - g(\mathbf{0}) & \leq \inprod{ \nabla g(\mathbf{0}) }{\bbeta^{(\sfS)} } - \frac{m}{2} ||\bbeta^{(\sfS)}||_2^2 \\
& \leq \max_{\bv} \inprod{ \nabla g(\mathbf{0}) }{ \bv)}  -  \frac{m}{2} \| \bv \|_F^2,
\end{align*}

where $\bv $ is an arbitrary vector that has support only on $\sfS$. Optimizing the RHS over $\bv$ gives the desired result. 
\end{proof}

\section{Additional experiments}
Figure~\ref{fig:synthetic1} shows the performance of all algorithms on the following metrics: log likelihood (normalized with respect to a null model), generalization to new test measurements from the same true support parameter, area under ROC, and percentage of the true support recovered for $l=2$. Recall that Figure~\ref{fig:synthetic2} presents the results from the same experiment with $l=10$. Clearly, the greedy algorithms benefit more from increased number of partitions.

\begin{figure}[ht]
\centering
%\vspace{.3in}
%\centerline{\fbox{This figure intentionally left non-blank}}
%\includegraphics[scale=0.4,clip=true,trim= 0 0 0 0.28in]{fig1dist6_N800_L2_c1_ccor05_numIt10_D.pdf}
\begin{subfigure}{0.45\columnwidth}
\includegraphics[scale=0.4]{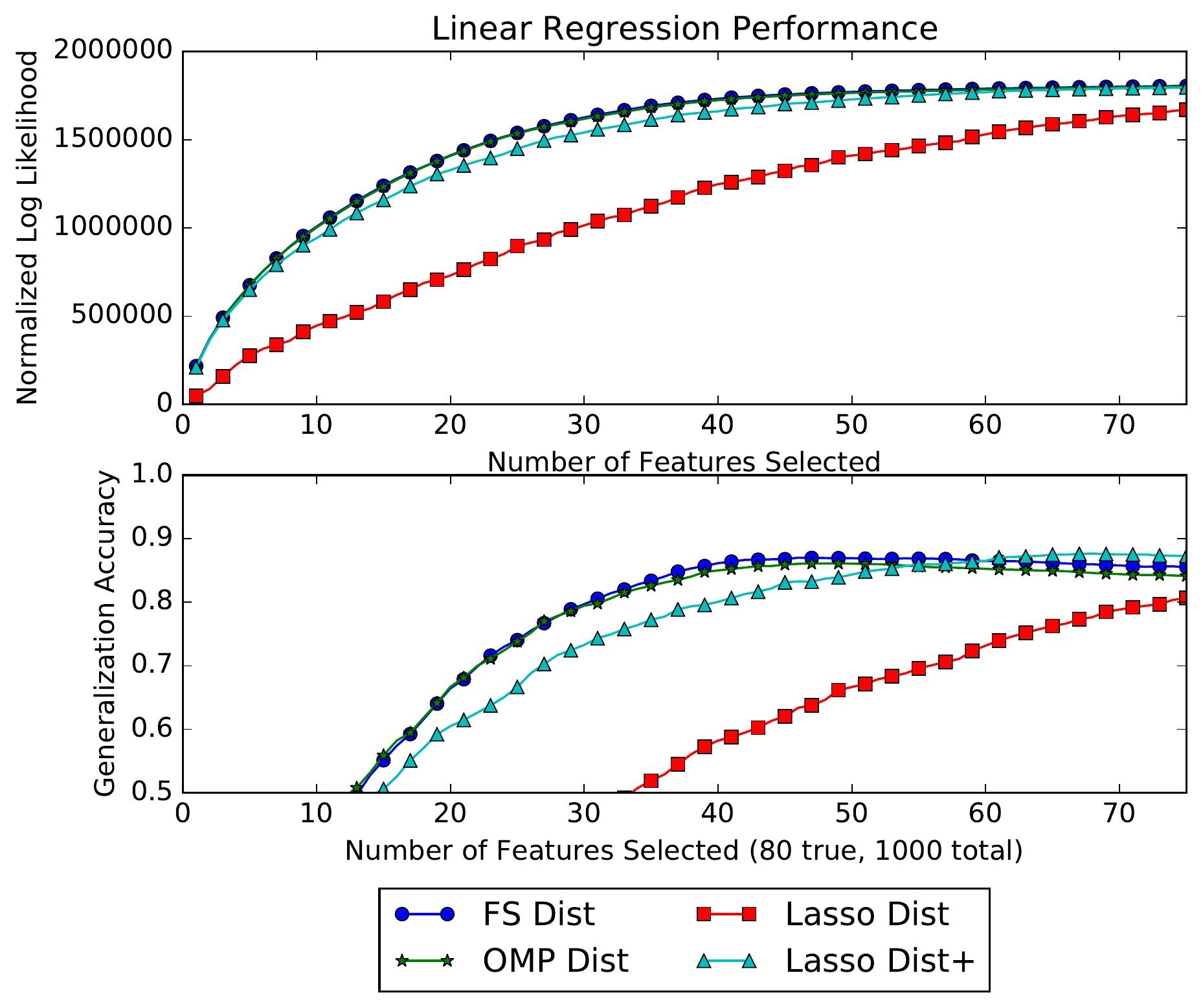}
%\caption{Distributed linear regression, $l=2$ partitions, $n=800$ training and test samples, $\alpha=0.5$. Training/testing performance}
\caption{}
\end{subfigure}
\begin{subfigure}{0.45\columnwidth}
\includegraphics[scale=0.4]{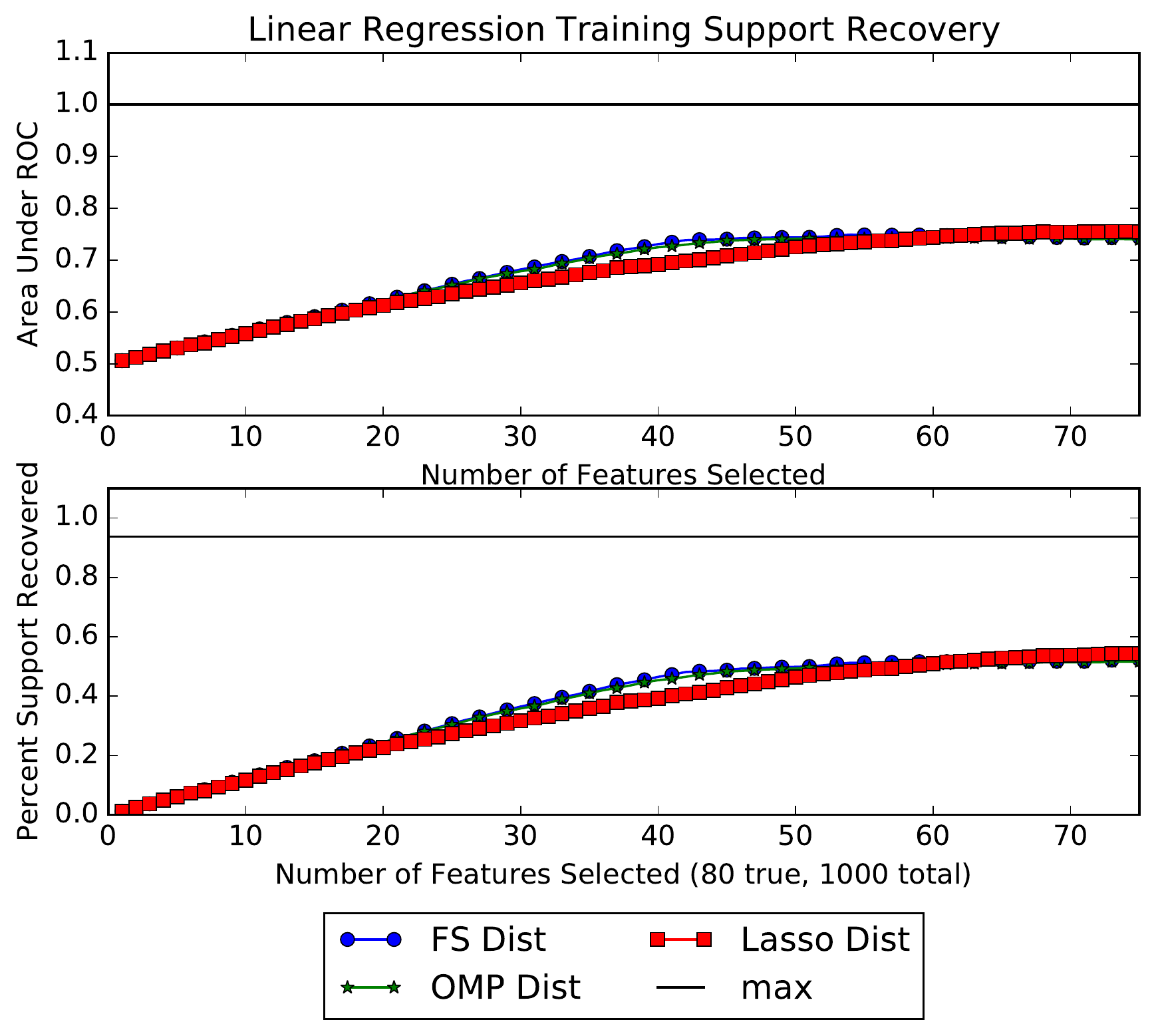}
%\vspace{.3in}
\caption{}
\end{subfigure}
\caption{\label{fig:synthetic1} Distributed linear regression, $l=2$ partitions, $n=800$ training and test samples, $\alpha=0.5$}
\end{figure}

%
%\begin{figure}[ht]
%\centering
%%\vspace{.3in}
%%\centerline{\fbox{This figure intentionally left non-blank}}
%\begin{subfigure}{0.9\columnwidth}
%%\includegraphics[scale=0.4,clip=true,trim= 0 0 0 0.28in]{fig1dist6_N800_L2_c1_ccor05_numIt10_D.pdf}
%\includegraphics[scale=0.4]{fig1dist6_N800_L2_c1_ccor05_numIt10_D.pdf}
%\caption{}\label{fig:synthetic1ll}
%\end{subfigure}\\
%\begin{subfigure}{0.9\columnwidth}
%%\includegraphics[scale=0.4,clip=true,trim= 0 0 0 0.29in]{fig2dist6_N800_L2_c1_ccor05_numIt10_D.pdf}
%\includegraphics[scale=0.4]{fig2dist6_N800_L2_c1_ccor05_numIt10_D.pdf}
%\caption{}\label{fig:synthetic1supp}
%\end{subfigure}
%%\vspace{.3in}
%\caption{\label{fig:synthetic1} Distributed linear regression, $l=2$ partitions, $n=800$ training and test samples, $\alpha=0.5$. Results averaged over $10$ iterations.}
%\end{figure}

\end{document}